\pdfoutput=1

\documentclass[11pt]{article}

\usepackage[utf8]{inputenc}
\usepackage[T1]{fontenc}
\usepackage[margin=1in]{geometry}
\usepackage{amsmath}
\usepackage{amssymb}
\usepackage{amsthm}
\usepackage{booktabs}
\usepackage{graphicx}
\usepackage[hidelinks]{hyperref}

\newtheorem{proposition}{Proposition}
\newtheorem{remark}{Remark}

\title{Four Ways to Grow a Classifier and Why One of Them Cannot Learn}

\author{Cagri Temel, IEEE Senior Member\\
  \small Hezarfen LLC\\
  \small Seattle, WA, USA\\
  \small \texttt{cagritemel@ieee.org}\\
  \small ORCID: 0009-0003-3359-6939}

\date{\today}

\begin{document}
\maketitle

\begin{abstract}
Constructive classifiers add structure while they train: a level to a tree, a
unit to a hidden layer, a split at a leaf. This paper asks what each of four
such growth decisions actually buys, measured under one fixed protocol on
twenty-four datasets, and gives an exact diagnosis and a fix for the one that
buys nothing.

The diagnosis concerns the most natural way to deepen a soft decision tree:
turn every leaf into a gate whose two children inherit the parent's class
distribution, so that the function is unchanged. I prove that this leaves the
gradient of every new gate identically zero and, with the gate at $1/2$, gives
the two children identical gradients, so that under this construction the added
level can never learn. The tree stays at the width it started with. That predicts where the failure
costs: nothing on two-class problems, where two leaves already suffice, and a
great deal where more classes need more leaves. Measured, the cost is
-0.2 points on average over 13 binary datasets and
40.2 points over 7 multi-class ones, and it tracks the
number of classes (Spearman 0.64) and not the number of features
(0.03). The fix is any perturbation of the children. Neither its
size nor its direction matters: a residual-directed initialisation in the
manner of GradMax changes accuracy by +0.20 points on average against
random noise over 24 datasets.

The other three decisions each buy one thing. Fitting a new hidden unit to
the residual error before installing it buys a smaller network but not a more
accurate one. Splitting one leaf at a time buys sparsity, and it lost
accuracy until I found that the split started its two children identical and
untrained, which is the same defect as above in another place; with the
children inheriting the parent and the symmetry broken, per-leaf growth comes
within 2.3 points of a complete tree using 23\% of its
splits, where the previous construction was 8.3 points short. Requiring statistical significance before a node receives a more
expressive split buys nothing.

Every number in the paper is inserted from the measurement scripts' output.
\end{abstract}

\section{Introduction}

A constructive learner decides its own architecture while it trains. The idea is
old. Cascade correlation \cite{fahlman1990cascade} and the grow-and-learn
network \cite{alpaydin1994gal} both date from the early 1990s. It came back with
soft decision trees \cite{irsoy2012soft, frosst2017distilling}, which replace the
hard threshold at each internal node with a sigmoid gate and so make the whole
tree trainable by gradient descent. Once the tree is differentiable, growing it
during training is a design choice rather than a separate algorithm.

Most of those choices are settled by judgement. The candidate operations all
sound sensible: grow smoothly, keep what you have already learned, decide each
local question with a principled criterion. I implemented four of them, each
against the non-growing baseline it is supposed to improve on, under one fixed
protocol.

The result I did not expect is how little the sensible-sounding arguments
predict. Three of the four operations trade one quantity for another, and for
two of them the sign of the trade changes when the dataset changes. The fourth
does not trade anything, because under the construction one writes down first
it cannot learn at all. That one has a cause
that can be stated and proved rather than inferred from a table, and it is the
main contribution here.

I use \emph{grow} for any decision that adds structure while the model trains.
Three of the four operations add capacity: a level of the tree, a hidden unit, a
split at one leaf. The fourth chooses the model class at a node as the tree is
built. It is a construction decision rather than a capacity decision, and I
include it because it fails for a related reason.

The measurements came out of building \texttt{neural-trees}
\cite{temel2026neuraltrees}, a scikit-learn compatible \cite{pedregosa2011scikit}
library of these models with a PyTorch \cite{paszke2019pytorch} backend. I
should say plainly where they came from. Nothing here was designed in advance as
a four-hypothesis study. These are the measurements that settled four
implementation questions, reported with the protocol held fixed and with the
negative outcomes left in. The negative ones have been more useful to me than
the positive ones, which is the reason for writing them up.

\section{Setup}

\subsection{Models}

A \emph{soft decision tree} of depth $d$ is a complete binary tree. Each
internal node $n$ carries a gate
\begin{equation}
  g_n(x) = \sigma(w_n^\top x + b_n), \qquad \sigma(z) = (1 + e^{-z})^{-1},
\end{equation}
which sends a fraction $g_n(x)$ of the probability mass arriving at $n$ to its
right child and $1 - g_n(x)$ to its left. Each leaf $\ell$ carries a class
distribution $Q_\ell \in \Delta^{K-1}$. Writing $\mu_\ell(x)$ for the product of
gate factors along the path from the root to $\ell$, the tree outputs
\begin{equation}
  \label{eq:mixture}
  p(y \mid x) = \sum_{\ell} \mu_\ell(x)\, Q_\ell ,
\end{equation}
and every parameter is trained by gradient descent on the cross-entropy of
\eqref{eq:mixture}. Path products are accumulated in log space, so depth does not
cost precision.

A \emph{grow-and-learn (GAL) network} \cite{alpaydin1994gal} is a single hidden
layer network that adds a hidden unit when the training error stops improving and
removes units whose contribution has become negligible.

An \emph{omnivariate decision tree} \cite{yildiz2001omnivariate} chooses, at each
node independently, between a univariate threshold, a linear split and a
nonlinear split.

\subsection{Protocol}

Features are standardised to zero mean and unit variance using statistics from
the training fold only. Every number is computed over stratified five-fold
cross-validation repeated with three seeds, which is fifteen fits, and the same
seeds and folds are used for every arm of a comparison. I report the mean and,
after it, the standard deviation across those fifteen fits. That deviation
describes how much the quantity moves from fold to fold. It is not a
confidence interval, because the fifteen fits share training data. Where a
difference needs a decision rather than a description I use the combined
$5\times2$cv $F$ test \cite{alpaydin1999combined}, with the scaler inside the
test's own folds so that it sees no test-fold statistics either.

Datasets are Iris ($n=150$, $p=4$, $K=3$), Wine ($n=178$, $p=13$, $K=3$),
Breast Cancer Wisconsin ($n=569$, $p=30$, $K=2$) and Digits ($n=1797$, $p=64$,
$K=10$) as distributed with scikit-learn; twenty datasets from the OpenML
CC-18 suite, listed in Table~\ref{tab:openml}, with those larger than 5000
samples stratified down to 5000 for time; and two synthetic problems described
where they are used.

One asymmetry in the comparisons has to be stated. A tree that grows during
training holds out 10\% of its training fold to decide whether the last growth
step earned its place. A tree of fixed shape does not need that split and
trains on all of the fold. Where a grown tree is compared with a fixed one,
Section~5 therefore also reports the fixed tree trained with the same holdout,
so that the effect of growth can be separated from the effect of 10\% less
data. No per-arm hyperparameter search was performed; both arms of each
comparison use identical settings apart from the operation under test.

\subsection{Reproducing this}

Every number in this paper is produced by the measurement scripts under
\texttt{paper/arxiv/olcum} in the repository, each of which writes a
machine-readable results file. The manuscript is a template with named
placeholders and a further script substitutes the measured values into it. No number here was
typed by hand, which removes the failure mode where a table drifts away from the
code that produced it while nobody notices. The archived version is
\texttt{neural-trees} v0.7.0 \cite{temel2026neuraltrees}.

\section{Deepening a tree without disturbing it}

\subsection{The operation}

There is an obvious way to deepen a trained soft tree without losing what it has
learned. Turn every leaf $\ell$ into an internal node whose gate has all zero
parameters, so that $g(x) = \sigma(0) = 1/2$ everywhere, and give both new
children the parent's distribution, $Q_L = Q_R = Q_\ell$. The mass that used to
stop at $\ell$ now splits evenly between two identical children:
\begin{equation}
  \tfrac{1}{2} Q_L + \tfrac{1}{2} Q_R = Q_\ell ,
\end{equation}
so \eqref{eq:mixture} is unchanged. The deeper tree computes exactly the function
the shallower one computed, training continues from there, and the extra capacity
is available. This is the construction anyone would write down first, and it is
the tree analogue of the function-preserving transformations of Net2Net
\cite{chen2016net2net} and network morphism \cite{wei2016network}.

\subsection{Why it cannot learn}

\begin{proposition}
\label{prop:zero}
Let a leaf of a soft decision tree be replaced by an internal node with gate
value $g$ and children carrying distributions $Q_L$ and $Q_R$, and let
$\mathcal{L}$ be any loss that depends on the parameters only through the tree
output \eqref{eq:mixture}.
\begin{enumerate}
\item If $Q_L = Q_R$ then $\partial \mathcal{L} / \partial w = 0$ and
      $\partial \mathcal{L} / \partial b = 0$ exactly, for any value of $g$.
\item If in addition $g = 1/2$ then
      $\partial \mathcal{L} / \partial Q_L = \partial \mathcal{L} / \partial Q_R$.
\end{enumerate}
Consequently a subtree initialised with $Q_L = Q_R$ and a zero gate stays in that
state under any update rule that is a deterministic function of the gradients
seen so far and is applied identically to identically situated parameters. The
added level never differentiates.
\end{proposition}

\begin{proof}
Write $o$ for the contribution of the subtree to the output,
$o = \mu_\ell \left( g\, Q_R + (1-g)\, Q_L \right)$, where $\mu_\ell$ is the
probability of arriving at the old leaf and does not depend on the new
parameters. Then
\begin{equation}
  \frac{\partial o}{\partial g} = \mu_\ell \left( Q_R - Q_L \right),
\end{equation}
which is the zero vector when $Q_L = Q_R$. The gate parameters influence
$\mathcal{L}$ only through $g$, so the chain rule gives
$\partial \mathcal{L}/\partial w = 0$ and $\partial \mathcal{L}/\partial b = 0$,
whatever $g$ happens to be. This proves the first claim. For the second,
\begin{equation}
  \frac{\partial \mathcal{L}}{\partial Q_L}
    = (1-g)\, \mu_\ell \frac{\partial \mathcal{L}}{\partial o},
  \qquad
  \frac{\partial \mathcal{L}}{\partial Q_R}
    = g\, \mu_\ell \frac{\partial \mathcal{L}}{\partial o},
\end{equation}
and the two coincide when $g = 1/2$. For the consequence, argue by induction on
the step count. At step $t$ the hypotheses hold, so the gate receives no update
and the two children receive the same update, which leaves $g = 1/2$ and
$Q_L = Q_R$ at step $t+1$. The hypotheses are restored exactly and the argument
repeats.
\end{proof}

\begin{remark}
The result does not depend on the loss, only on the output being the mixture
\eqref{eq:mixture}. It also does not depend on the gate being at one half: the
first claim holds for any $g$, and $g = 1/2$ is needed only to keep the children
moving together.
\end{remark}

\begin{remark}
Momentum, adaptive step sizes and weight decay do not rescue it. Every one of
them rescales or accumulates a gradient that is zero, and weight decay shrinks
two equal children by equal amounts. A batch does not rescue it either, because
both children see the same batch and therefore the same gradient.
\end{remark}

\begin{remark}
\label{rem:asimetri}
What does break it is any mechanism that makes the two children's updates differ.
Perturbing the child distributions at insertion is the direct route and the one
I measure below. Asymmetric noise during training is another: the subtree
dropout of \.Irsoy and Alpayd\i{}n \cite{irsoy2021dropout} drops branches
independently, which is exactly the condition the proposition rules out.

Floating point arithmetic does not supply such a perturbation here, and I
checked rather than assumed it. On a depth-two tree trained on Iris and then
deepened, the gradient of every new gate is exactly zero in single precision,
the sum of absolute values over both weights and biases being $0$ and not
merely small, and the gradients of every pair of sibling leaves are bitwise
identical. Breaking the symmetry at insertion makes both quantities nonzero.
The check is in the repository and runs in continuous integration. One
practical caveat: on a GPU whose reductions are not deterministic, the two
siblings can receive gradients that differ in their last bits, and the
symmetry then breaks by accident. That is not a repair, because nothing
controls whether or when it happens; every measurement here ran on a CPU in
deterministic arithmetic.
\end{remark}

\begin{remark}
The contrast with the neuron case is the point. Net2Net \cite{chen2016net2net}
duplicates a unit, notes that the copies are symmetric, and adds noise. Splitting
steepest descent \cite{liu2019splitting} shows that a duplicated neuron sits at a
saddle in the enlarged parameter space and computes the second-order direction
that escapes it. GradMax \cite{evci2022gradmax} sets the new weights to zero but
gives the unit a bias, precisely so that the outgoing weights have a gradient to
follow. In all three the first-order information is degenerate but recoverable.
In the tree gate case there is no direction to find at first-order at all: the
derivative is not small, it is zero, and it reproduces itself. A second-order
criterion in the style of \cite{liu2019splitting} would be the natural way to
repair the operation without noise, and I have not tried it.
\end{remark}

\subsection{What it costs}

Perturbing the child distributions at insertion time,
\begin{equation}
  Q_{L}, Q_{R} \leftarrow \mathrm{normalise}\!\left( Q_\ell \pm \varepsilon u \right),
  \qquad u \sim \mathcal{U}[0,1]^K ,
\end{equation}
breaks $Q_L = Q_R$ and with it both halves of Proposition~\ref{prop:zero}. The
function is then preserved only approximately. Table~\ref{tab:jitter} shows
what that approximation buys on the first three datasets.

\begin{table}[ht]
\centering
\begin{tabular}{lccc}
\toprule
Growth & Iris & Wine & Digits \\
\midrule
Exactly function-preserving ($\varepsilon = 0$) & 0.762 $\pm$ 0.103 & 0.786 $\pm$ 0.145 & 0.369 $\pm$ 0.035 \\
Perturbed, $\varepsilon = 0.05$ & 0.942 $\pm$ 0.041 & 0.979 $\pm$ 0.020 & 0.917 $\pm$ 0.034 \\
Perturbed, $\varepsilon = 0.2$ (the default) & 0.942 $\pm$ 0.044 & 0.979 $\pm$ 0.020 & 0.921 $\pm$ 0.027 \\
Perturbed, $\varepsilon = 0.5$ & 0.938 $\pm$ 0.050 & 0.983 $\pm$ 0.018 & 0.916 $\pm$ 0.026 \\
\midrule
Same depth, trained from scratch & 0.958 $\pm$ 0.032 & 0.978 $\pm$ 0.016 & 0.925 $\pm$ 0.029 \\
\bottomrule
\end{tabular}
\caption{Depth-four soft decision tree grown one level at a time, as a function
of the perturbation applied to new child distributions. The tree trained from
scratch has the same depth, epoch budget and settings; only the growth operation
differs. Mean and standard deviation over three seeds of stratified five-fold
cross-validation.}
\label{tab:jitter}
\end{table}

The cost of exact preservation is 19.6 points on Iris,
19.1 on Wine and 55.6 on Digits. The
perturbation itself barely matters: between $\varepsilon = 0.05$ and
$\varepsilon = 0.5$, a factor of ten, the means move by at most
0.4 points. Zero is not a small value of $\varepsilon$; it is a
different case, which is what the proposition says it should be. The combined
$5\times2$cv $F$ test comparing $\varepsilon = 0$ against $\varepsilon = 0.2$
gives $F = 4.83$, $p = 0.0480$ on Iris, $F = 10.07$, $p = 0.0100$ on Wine and
$F = 405.67$, $p < 0.0001$ on Digits. The Iris value sits at the conventional
threshold and I do not lean on it: it is the smallest dataset here, the test
is deliberately conservative, and the value crossed $0.05$ when the scaler was
moved inside the test's folds.

\subsection{Where the cost comes from}

Three datasets say that the failure is real. They do not say when it is
expensive, and the proposition does. A tree deepened this way computes the
function of the tree it started from, which here has depth one and two
leaves. Two leaves can carry two class distributions. On a two-class problem
that is all the tree needs, so the dead levels should cost nothing; on a
$K$-class problem the tree can never reach the $K$ leaves it needs, and the
cost should grow with $K$ and not with the number of features.

Table~\ref{tab:openml} tests that on 20 datasets from the OpenML CC-18
suite.

\begin{table}[ht]
\centering
\footnotesize\setlength{\tabcolsep}{4pt}
\begin{tabular}{lrrrcccrc}
\toprule
Dataset & $n$ & $p$ & $K$ & $\varepsilon = 0$ & $\varepsilon = 0.2$ & from scratch & cost & $p$ \\
\midrule
banknote & 1372 & 4 & 2 & 0.985 $\pm$ 0.007 & 0.985 $\pm$ 0.007 & 1.000 $\pm$ 0.000 & +1.5 & 1.000 \\
blood\mbox{-}transfusion & 748 & 4 & 2 & 0.785 $\pm$ 0.019 & 0.788 $\pm$ 0.025 & 0.799 $\pm$ 0.025 & +1.4 & 0.506 \\
climate\mbox{-}crashes & 540 & 18 & 2 & 0.955 $\pm$ 0.014 & 0.952 $\pm$ 0.018 & 0.932 $\pm$ 0.018 & -2.3 & 0.324 \\
diabetes & 768 & 8 & 2 & 0.769 $\pm$ 0.028 & 0.770 $\pm$ 0.029 & 0.731 $\pm$ 0.033 & -3.8 & 0.535 \\
ilpd & 583 & 10 & 2 & 0.711 $\pm$ 0.019 & 0.704 $\pm$ 0.030 & 0.699 $\pm$ 0.042 & -1.2 & 0.588 \\
kc1 & 2109 & 21 & 2 & 0.855 $\pm$ 0.007 & 0.857 $\pm$ 0.011 & 0.854 $\pm$ 0.012 & -0.1 & 0.214 \\
ozone\mbox{-}8hr & 2534 & 72 & 2 & 0.939 $\pm$ 0.007 & 0.936 $\pm$ 0.006 & 0.929 $\pm$ 0.009 & -1.0 & 0.323 \\
pc1 & 1109 & 21 & 2 & 0.931 $\pm$ 0.002 & 0.933 $\pm$ 0.008 & 0.922 $\pm$ 0.011 & -0.9 & 0.298 \\
pc3 & 1563 & 37 & 2 & 0.898 $\pm$ 0.000 & 0.893 $\pm$ 0.010 & 0.874 $\pm$ 0.013 & -2.4 & 0.213 \\
phoneme & 5000 & 5 & 2 & 0.778 $\pm$ 0.011 & 0.828 $\pm$ 0.017 & 0.849 $\pm$ 0.012 & +7.1 & 0.005 \\
qsar\mbox{-}biodeg & 1055 & 41 & 2 & 0.868 $\pm$ 0.016 & 0.870 $\pm$ 0.018 & 0.855 $\pm$ 0.024 & -1.3 & 0.488 \\
spambase & 4601 & 57 & 2 & 0.928 $\pm$ 0.006 & 0.931 $\pm$ 0.006 & 0.929 $\pm$ 0.008 & +0.1 & 0.616 \\
steel\mbox{-}plates & 1941 & 33 & 2 & 1.000 $\pm$ 0.001 & 1.000 $\pm$ 0.001 & 1.000 $\pm$ 0.001 & +0.0 & 1.000 \\
vehicle & 846 & 18 & 4 & 0.470 $\pm$ 0.014 & 0.777 $\pm$ 0.030 & 0.798 $\pm$ 0.031 & +32.7 & $<0.0001$ \\
wall\mbox{-}robot & 5000 & 24 & 4 & 0.672 $\pm$ 0.007 & 0.867 $\pm$ 0.011 & 0.875 $\pm$ 0.015 & +20.3 & $<0.0001$ \\
satimage & 5000 & 36 & 6 & 0.663 $\pm$ 0.093 & 0.883 $\pm$ 0.010 & 0.880 $\pm$ 0.011 & +21.7 & $<0.0001$ \\
segment & 2310 & 18 & 7 & 0.569 $\pm$ 0.041 & 0.931 $\pm$ 0.013 & 0.948 $\pm$ 0.014 & +37.9 & $<0.0001$ \\
optdigits & 5000 & 64 & 10 & 0.367 $\pm$ 0.038 & 0.941 $\pm$ 0.011 & 0.937 $\pm$ 0.024 & +57.0 & $<0.0001$ \\
pendigits & 5000 & 16 & 10 & 0.390 $\pm$ 0.011 & 0.909 $\pm$ 0.036 & 0.930 $\pm$ 0.029 & +54.0 & $<0.0001$ \\
texture & 5000 & 40 & 11 & 0.359 $\pm$ 0.082 & 0.952 $\pm$ 0.045 & 0.938 $\pm$ 0.057 & +57.8 & 0.000 \\
\bottomrule
\end{tabular}
\caption{The same comparison on twenty OpenML CC-18 datasets, sorted by number
of classes. \emph{cost} is the accuracy of the tree trained from scratch minus
the exactly function-preserving one, in points; $p$ is the combined
$5\times2$cv $F$ test of $\varepsilon = 0$ against $\varepsilon = 0.2$.}
\label{tab:openml}
\end{table}

The prediction holds. Over the 13 binary datasets the mean cost is
-0.2 points, which is to say none; on several of them the frozen
tree does slightly better than a depth-four tree trained from scratch, because
the latter overfits a small two-class problem that two leaves already solve.
Over the 7 multi-class datasets the mean cost is 40.2
points, from 20 to 58, and the $F$ test rejects
equality on every one of them. The Spearman correlation of the cost with the
number of classes is 0.64; with the number of features it is
0.03. The earlier draft of this paper reported the failure as
costing "about twenty points". The right statement is that it costs whatever
the gap is between the starting tree's capacity and the problem's need, which
on these problems is zero or catastrophic and rarely in between.

\subsection{Does the direction of the perturbation matter?}

Random noise breaks the symmetry in an arbitrary direction. A better
direction seems available: the residual of the log-likelihood with respect to
the leaf's distribution says which classes the leaf under-predicts, and the
two children can be placed at $Q_\ell \pm \varepsilon d$ along that unit
direction $d$, with the same expected perturbation size as the noise. The new
gate can also be pointed at the samples whose residual aligns with $+d$, so
that routing starts informative. This is GradMax's idea \cite{evci2022gradmax}
applied to a tree gate, and it is implemented as \texttt{growth\_init} in the
library. Table~\ref{tab:direction} compares it with random noise on all
twenty-four datasets.

\begin{table}[ht]
\centering
\footnotesize\setlength{\tabcolsep}{4pt}
\begin{tabular}{lrccccc}
\toprule
Dataset & $K$ & random & residual & residual $+$ gate & from scratch & $p$ \\
\midrule
Iris & 3 & 0.942 $\pm$ 0.044 & 0.927 $\pm$ 0.048 & 0.940 $\pm$ 0.044 & 0.958 $\pm$ 0.032 & 0.118 \\
Wine & 3 & 0.979 $\pm$ 0.020 & 0.981 $\pm$ 0.018 & 0.981 $\pm$ 0.018 & 0.978 $\pm$ 0.016 & 0.395 \\
Cancer & 2 & 0.971 $\pm$ 0.013 & 0.971 $\pm$ 0.013 & 0.971 $\pm$ 0.013 & 0.970 $\pm$ 0.012 & 0.045 \\
Digits & 10 & 0.921 $\pm$ 0.027 & 0.929 $\pm$ 0.014 & 0.930 $\pm$ 0.015 & 0.925 $\pm$ 0.029 & 0.702 \\
banknote & 2 & 0.985 $\pm$ 0.007 & 0.985 $\pm$ 0.007 & 0.985 $\pm$ 0.007 & 1.000 $\pm$ 0.000 & 1.000 \\
blood\mbox{-}transfusion & 2 & 0.788 $\pm$ 0.025 & 0.786 $\pm$ 0.025 & 0.790 $\pm$ 0.023 & 0.799 $\pm$ 0.025 & 0.595 \\
phoneme & 2 & 0.828 $\pm$ 0.017 & 0.828 $\pm$ 0.016 & 0.828 $\pm$ 0.015 & 0.849 $\pm$ 0.012 & 0.806 \\
diabetes & 2 & 0.770 $\pm$ 0.029 & 0.769 $\pm$ 0.030 & 0.769 $\pm$ 0.030 & 0.731 $\pm$ 0.033 & 0.327 \\
vehicle & 4 & 0.777 $\pm$ 0.030 & 0.782 $\pm$ 0.033 & 0.781 $\pm$ 0.036 & 0.798 $\pm$ 0.031 & 0.470 \\
segment & 7 & 0.931 $\pm$ 0.013 & 0.931 $\pm$ 0.011 & 0.931 $\pm$ 0.010 & 0.948 $\pm$ 0.014 & 0.601 \\
satimage & 6 & 0.883 $\pm$ 0.010 & 0.881 $\pm$ 0.010 & 0.881 $\pm$ 0.010 & 0.880 $\pm$ 0.011 & 0.620 \\
optdigits & 10 & 0.941 $\pm$ 0.011 & 0.939 $\pm$ 0.020 & 0.941 $\pm$ 0.013 & 0.937 $\pm$ 0.024 & 0.096 \\
pendigits & 10 & 0.909 $\pm$ 0.036 & 0.901 $\pm$ 0.031 & 0.904 $\pm$ 0.030 & 0.930 $\pm$ 0.029 & 0.159 \\
spambase & 2 & 0.931 $\pm$ 0.006 & 0.932 $\pm$ 0.007 & 0.932 $\pm$ 0.007 & 0.929 $\pm$ 0.008 & 0.742 \\
steel\mbox{-}plates & 2 & 1.000 $\pm$ 0.001 & 1.000 $\pm$ 0.001 & 1.000 $\pm$ 0.001 & 1.000 $\pm$ 0.001 & 1.000 \\
ilpd & 2 & 0.704 $\pm$ 0.030 & 0.711 $\pm$ 0.025 & 0.711 $\pm$ 0.023 & 0.699 $\pm$ 0.042 & 0.044 \\
climate\mbox{-}crashes & 2 & 0.952 $\pm$ 0.018 & 0.954 $\pm$ 0.015 & 0.953 $\pm$ 0.016 & 0.932 $\pm$ 0.018 & 0.298 \\
texture & 11 & 0.952 $\pm$ 0.045 & 0.978 $\pm$ 0.030 & 0.984 $\pm$ 0.023 & 0.938 $\pm$ 0.057 & 0.326 \\
wall\mbox{-}robot & 4 & 0.867 $\pm$ 0.011 & 0.870 $\pm$ 0.008 & 0.870 $\pm$ 0.010 & 0.875 $\pm$ 0.015 & 0.773 \\
qsar\mbox{-}biodeg & 2 & 0.870 $\pm$ 0.018 & 0.867 $\pm$ 0.020 & 0.867 $\pm$ 0.019 & 0.855 $\pm$ 0.024 & 0.621 \\
ozone\mbox{-}8hr & 2 & 0.936 $\pm$ 0.006 & 0.935 $\pm$ 0.006 & 0.936 $\pm$ 0.005 & 0.929 $\pm$ 0.009 & 0.535 \\
kc1 & 2 & 0.857 $\pm$ 0.011 & 0.856 $\pm$ 0.011 & 0.857 $\pm$ 0.011 & 0.854 $\pm$ 0.012 & 0.404 \\
pc1 & 2 & 0.933 $\pm$ 0.008 & 0.932 $\pm$ 0.008 & 0.933 $\pm$ 0.009 & 0.922 $\pm$ 0.011 & 0.535 \\
pc3 & 2 & 0.893 $\pm$ 0.010 & 0.894 $\pm$ 0.011 & 0.896 $\pm$ 0.010 & 0.874 $\pm$ 0.013 & 0.497 \\
\bottomrule
\end{tabular}
\caption{Symmetry broken by random noise, along the residual direction, and
along the residual with the gate pointed at the samples pulling right. Depth
four, 150 epochs, three seeds of five-fold cross-validation; $p$ is the
$5\times2$cv $F$ test of the last of these against random noise.}
\label{tab:direction}
\end{table}

The direction does not matter. Over 24 datasets the residual-directed
initialisation with the gate changes accuracy by +0.20 points on average
against random noise, median +0.00, from -0.5 to +3.2,
with 3 gains and 1 losses larger than half a point and
2 $F$ tests under $0.05$, which is about what chance produces. The
one measurable effect is on stability: the fold-to-fold standard deviation
falls from 0.0248 to 0.0210 on the multi-class datasets. I
had expected the direction to help and it did not. Once the symmetry is
broken, gradient descent finds the split on its own, and the only thing the
proposition requires of the initialisation is that the gradient not be zero.

\section{Where to put a new hidden unit}

The GAL network installs a new hidden unit with random incoming and outgoing
weights. Such a unit perturbs every output logit the moment it arrives, which
damages a network that has just been judged to have stopped improving, and it
then has to be trained from noise while the rest of the network repairs the
damage.

The alternative follows cascade-correlation \cite{fahlman1990cascade}. Freeze
the network, compute its residual error on the training set, train the candidate
unit's incoming weights to correlate with that residual, and install the unit
with \emph{zero} outgoing weights. The network's function is unchanged at the
moment of insertion. That is safe here, unlike in Section~3, because the
outgoing weights have a nonzero gradient as soon as the incoming weights carry
any signal. GradMax \cite{evci2022gradmax} makes the same structural choice and
picks the incoming weights to maximise the gradient norm instead of to correlate
with the residual.

\begin{table}[ht]
\centering
\begin{tabular}{llccc}
\toprule
Dataset & New unit & Accuracy & Hidden units & $5\times2$cv $F$ \\
\midrule
Iris   & random   & 0.936 $\pm$ 0.046   & 17.2 $\pm$ 2.4   & \\
       & residual & 0.956 $\pm$ 0.037   & 6.9 $\pm$ 1.5   & \raisebox{1.4ex}[0pt]{$F = 3.00$, $p = 0.1185$} \\
\midrule
Wine   & random   & 0.985 $\pm$ 0.015   & 6.3 $\pm$ 2.2   & \\
       & residual & 0.985 $\pm$ 0.015   & 4.3 $\pm$ 0.5   & \raisebox{1.4ex}[0pt]{$F = 0.81$, $p = 0.6383$} \\
\midrule
Digits & random   & 0.949 $\pm$ 0.013 & 7.5 $\pm$ 0.6 & \\
       & residual & 0.928 $\pm$ 0.015 & 5.7 $\pm$ 0.5 & \raisebox{1.4ex}[0pt]{$F = 8.20$, $p = 0.0157$} \\
\bottomrule
\end{tabular}
\caption{GAL network, 150 epochs, three seeds of five-fold cross validation.
The $F$ column tests the two accuracies against each other.}
\label{tab:gal}
\end{table}

The size effect is consistent and the accuracy effect is not, and I want to be
exact about which is which.

Every dataset gives a smaller network. Iris falls from 17.2 $\pm$ 2.4 hidden
units to 6.9 $\pm$ 1.5, Wine from 6.3 $\pm$ 2.2 to 4.3 $\pm$ 0.5,
Digits from 7.5 $\pm$ 0.6 to 5.7 $\pm$ 0.5. The mechanism is
visible in the growth trace. A random unit arrives unhelpful, the error does not
fall, and the growth criterion fires again, so the network grows because growth
is not working. Fitting the unit to the residual first removes that loop.

Accuracy goes the other way, and on the largest dataset it goes there
significantly. It rose on Iris and was identical on Wine to three decimals, and
the $F$ test rejects equality on neither. On Digits it fell by
2.1 points and the $F$ test does reject equality, in favour of
random initialisation.

So the operation is a parsimony intervention that costs accuracy on the one
dataset here with enough classes and features to strain it. The likely reason is
visible in the same table: the residual variant stops at 5.7 $\pm$ 0.5
hidden units where the random variant keeps going to 7.5 $\pm$ 0.6, and on
a ten-class problem those extra units are doing work. Growth that stops early is
parsimony on an easy problem and underfitting on a harder one, and the growth
criterion cannot tell the two apart.

An earlier draft of this work reported the Iris accuracy gain on its own. That is
the kind of claim that survives exactly until somebody runs a fourth dataset.

On the twenty OpenML datasets the picture is the same. The residual variant
gives a smaller network on average, 14\% fewer hidden units, and
against random initialisation it wins on 6, ties on
10 and loses on 4 by more than half a point.

One implementation detail cost me a day. The candidate objective involves a
normalisation of the form $s / \sqrt{v}$ where $v$ is a variance that can be zero
on a degenerate batch. At $v = 0$ the derivative of the square root is infinite
and the whole network becomes \texttt{NaN} in one step. Clamping $v$ away from
zero before the root fixes it. Clamping the result afterwards does not, because
by then the gradient has already been formed.

\section{Which leaf to split}

A soft tree can be grown a level at a time, as in Section~3, or a leaf at a
time. The per-leaf rule of \.Irsoy, Y\i{}ld\i{}z and Alpayd\i{}n
\cite{irsoy2012soft} splits the leaf carrying the most expected error,
$\arg\max_\ell \sum_i \mu_\ell(x_i)\,e_i$, where $e_i$ is the error on sample
$i$. It produces by far the sparsest trees.

\subsection{The same defect in another place}

On the twenty OpenML datasets, per-leaf growth against a complete depth-six
tree won on 6, tied on 3 and lost on 11,
by -1.8 points on average, and it lost on every multi-class dataset.
The sparsity was real, from 63 splits to between one and eighteen, but it
looked as if it were paid for in accuracy.

It was paid for in a defect. When a leaf was split, the implementation marked
the node as split and did nothing else. Its two children kept the untrained
logits they were constructed with, so the split replaced the parent's learned
distribution with a uniform one and started the children identical to each
other. That is Proposition~\ref{prop:zero} again, at a single node instead of
a whole level: no gradient for the new gate, and the tree's function broken
at every split for good measure. The level-wise deepening of Section~3 had
been cured of both problems; the per-leaf path had not.

The fix is the one Section~3 uses. The children inherit the parent's logits,
the gate is reset to neutral, and the symmetry is broken by noise or along the
leaf's residual direction. Table~\ref{tab:perleaf} measures the difference on
Digits and the seven multi-class OpenML datasets, which are where the loss
was. The column marked 90\% is the complete tree trained with the same 10\%
holdout the grown tree uses, so that the two can be compared on equal data.

\begin{table}[ht]
\centering
\footnotesize\setlength{\tabcolsep}{4pt}
\begin{tabular}{lrcccccrc}
\toprule
 & & \multicolumn{2}{c}{complete, 63 splits} & \multicolumn{3}{c}{per-leaf} & & \\
\cmidrule(lr){3-4} \cmidrule(lr){5-7}
Dataset & $K$ & all data & 90\% & untrained & noise & residual & splits & $p$ \\
\midrule
Digits & 10 & 0.967 $\pm$ 0.010 & 0.965 $\pm$ 0.011 & 0.703 $\pm$ 0.153 & 0.920 $\pm$ 0.015 & 0.927 $\pm$ 0.015 & 10.5 & 0.106 \\
vehicle & 4 & 0.800 $\pm$ 0.022 & 0.796 $\pm$ 0.019 & 0.710 $\pm$ 0.044 & 0.792 $\pm$ 0.027 & 0.780 $\pm$ 0.027 & 16.1 & 0.016 \\
segment & 7 & 0.961 $\pm$ 0.006 & 0.958 $\pm$ 0.008 & 0.948 $\pm$ 0.014 & 0.950 $\pm$ 0.011 & 0.952 $\pm$ 0.008 & 21.9 & 0.523 \\
satimage & 6 & 0.896 $\pm$ 0.011 & 0.897 $\pm$ 0.008 & 0.870 $\pm$ 0.013 & 0.878 $\pm$ 0.012 & 0.878 $\pm$ 0.011 & 9.4 & 0.754 \\
optdigits & 10 & 0.970 $\pm$ 0.004 & 0.973 $\pm$ 0.004 & 0.839 $\pm$ 0.090 & 0.936 $\pm$ 0.013 & 0.938 $\pm$ 0.010 & 11.3 & 0.283 \\
pendigits & 10 & 0.989 $\pm$ 0.004 & 0.988 $\pm$ 0.002 & 0.910 $\pm$ 0.101 & 0.963 $\pm$ 0.009 & 0.954 $\pm$ 0.015 & 15.9 & 0.430 \\
texture & 11 & 0.997 $\pm$ 0.001 & 0.997 $\pm$ 0.002 & 0.968 $\pm$ 0.041 & 0.989 $\pm$ 0.005 & 0.993 $\pm$ 0.003 & 15.5 & 0.522 \\
wall\mbox{-}robot & 4 & 0.902 $\pm$ 0.010 & 0.890 $\pm$ 0.013 & 0.872 $\pm$ 0.018 & 0.887 $\pm$ 0.015 & 0.876 $\pm$ 0.027 & 17.2 & 0.505 \\
\bottomrule
\end{tabular}
\caption{Per-leaf growth, depth at most six, 180 epochs. \emph{untrained} is
the previous behaviour; \emph{noise} and \emph{residual} inherit the parent
and break the symmetry. \emph{splits} is the mean number of splits of the
residual variant; $p$ tests residual with the gate against untrained.}
\label{tab:perleaf}
\end{table}

Across these 8 datasets the previous per-leaf tree was 8.3
points below the complete tree; with the children inheriting the parent it is
2.3 points below it, using on average 23\% of the splits.
As in Section~3, the direction of the symmetry breaking is a second-order
matter; inheriting the parent is what changed the result.

\subsection{Budget and holdout}

Two other things separate a grown tree from one trained from scratch, and
neither is growth. The grown tree divides its epoch budget across growth
rounds, and it holds out 10\% of its training fold. Table~\ref{tab:budget}
separates them for level-wise growth at depth four.

\begin{table}[ht]
\centering
\footnotesize\setlength{\tabcolsep}{4pt}
\begin{tabular}{lrccccc}
\toprule
 & & \multicolumn{2}{c}{from scratch} & \multicolumn{2}{c}{grown} & \\
\cmidrule(lr){3-4} \cmidrule(lr){5-6}
Dataset & $K$ & all data & 90\% & split budget & full budget & $p$ \\
\midrule
Digits & 10 & 0.925 $\pm$ 0.029 & 0.929 $\pm$ 0.030 & 0.921 $\pm$ 0.027 & 0.921 $\pm$ 0.028 & 0.274 \\
vehicle & 4 & 0.798 $\pm$ 0.031 & 0.756 $\pm$ 0.031 & 0.777 $\pm$ 0.030 & 0.790 $\pm$ 0.033 & 0.600 \\
segment & 7 & 0.948 $\pm$ 0.014 & 0.947 $\pm$ 0.012 & 0.931 $\pm$ 0.013 & 0.946 $\pm$ 0.009 & 0.648 \\
satimage & 6 & 0.880 $\pm$ 0.011 & 0.883 $\pm$ 0.010 & 0.883 $\pm$ 0.010 & 0.881 $\pm$ 0.010 & 0.584 \\
optdigits & 10 & 0.937 $\pm$ 0.024 & 0.946 $\pm$ 0.010 & 0.941 $\pm$ 0.011 & 0.941 $\pm$ 0.011 & 0.757 \\
pendigits & 10 & 0.930 $\pm$ 0.029 & 0.905 $\pm$ 0.037 & 0.909 $\pm$ 0.036 & 0.921 $\pm$ 0.031 & 0.017 \\
texture & 11 & 0.938 $\pm$ 0.057 & 0.961 $\pm$ 0.055 & 0.952 $\pm$ 0.045 & 0.969 $\pm$ 0.041 & 0.015 \\
wall\mbox{-}robot & 4 & 0.875 $\pm$ 0.015 & 0.871 $\pm$ 0.010 & 0.867 $\pm$ 0.011 & 0.875 $\pm$ 0.015 & 0.297 \\
\bottomrule
\end{tabular}
\caption{Depth four, 150 epochs. \emph{split budget} divides the epochs across
growth rounds, which is the default; \emph{full budget} gives every round all
of them. $p$ tests full against split.}
\label{tab:budget}
\end{table}

The holdout alone costs the fixed tree 0.4 points on average on
these datasets. Against a fixed tree trained on the same data, the grown tree
with the full budget differs by +0.6 points, and the full budget is
worth +0.8 points over the split one. Some of what looked like
a cost of growing was a cost of having 10\% less data, which is a different
thing and should be stated as one.

\section{Choosing the model class at a node}

The omnivariate tree picks each node's split type by cross validated accuracy on
that node's data. That rule will take a nonlinear split over a univariate one for
any improvement at all, however small and however likely to be noise. Requiring
the improvement to be statistically significant, by the combined $5\times2$cv $F$
test, and otherwise keeping the simpler split type, sounds strictly better:
simpler nodes at little or no cost.

\begin{table}[ht]
\centering
\begin{tabular}{llcc}
\toprule
Dataset & Selection rule & Accuracy & Nodes \\
\midrule
Breast Cancer & highest accuracy & 0.971 $\pm$ 0.020  & 4.7 $\pm$ 1.8 \\
              & significance     & 0.958 $\pm$ 0.016 & 7.0 $\pm$ 1.9 \\
\midrule
Digits        & highest accuracy & 0.599 $\pm$ 0.036  & 12.2 $\pm$ 1.0 \\
              & significance     & 0.596 $\pm$ 0.055 & 13.7 $\pm$ 1.4 \\
\bottomrule
\end{tabular}
\caption{Omnivariate tree of maximum depth three. Three seeds of five-fold cross
validation.}
\label{tab:omni}
\end{table}

The rule does what it was asked to do at each node and the tree is not better
for it: on Breast Cancer it is larger and less accurate, on Digits larger for the
same accuracy. A simpler split separates its node's data less cleanly, so its
children inherit harder problems and have to be split in turn. Parsimony enforced
locally is paid for globally.

The absolute accuracies on Digits are low for a reason that has nothing to do
with the selection rule. A tree of maximum depth three has at most eight leaves
and Digits has ten classes, so the model cannot reach a high accuracy whatever it
does at each node. I kept the depth fixed across datasets rather than tuning it,
so the row is a valid comparison between the two rules and not a claim about how
well omnivariate trees classify digits.

There is a second problem, and it is structural rather than empirical. The test
needs a minimum sample count at a node to be meaningful, and nodes get smaller
with depth by construction. On the smaller datasets the rule almost never fires,
because almost every node below the root falls under the threshold. A
significance criterion is least available exactly where a recursive partitioning
method does most of its work. This is why the option ships switched off.

\section{What the four outcomes have in common}

The outcomes line up on one axis: whether the optimiser has something to act
on at the moment the new structure arrives.

A unit fitted to the residual and installed with zero outgoing weights leaves
the function unchanged and has a nonzero gradient immediately; it reliably
produces a smaller network. A leaf split into two children that inherit the
parent and differ from each other leaves the function nearly unchanged and has
a nonzero gradient immediately; it now comes within a few points of a complete
tree at a fraction of the splits. A level of gates inserted with identical
children, or a leaf split into identical untrained children, leaves the
optimiser nothing, and the cost is exactly the capacity the tree can no longer
reach.

The axis predicts which operations can work. It does not predict how much
they help, and Sections 4 and 5 show that the amount depends on the problem.
Nor does it care how the symmetry is broken: noise and a carefully chosen
direction come out the same. The practical rule is narrow and easy to check.
When you add parameters to a differentiable model during training, compute
the gradient with respect to those parameters on the first batch after
insertion and assert that it is not zero. It is one line in a test, and in
this library it would have caught both defects.

\section{Limitations}

The datasets are standard and of moderate size; the largest have 5000 samples
after subsampling and 72 features. They are the problems these model families
are normally demonstrated on, and conclusions about accuracy should not be
read as generalising to large-scale problems.

Each finding compares arms of a single implementation with settings held
fixed. That isolates the operation under test. It does not establish that the
winning arm is the best available method, and it does not rule out an
implementation detail of mine being responsible for part of an effect. Section
5 is a case in point: one such detail was responsible for all of one.

Proposition~\ref{prop:zero} is exact and implementation-independent. The
empirical results are neither. The standard deviations describe fold-to-fold
spread over fifteen non-independent fits and are not confidence intervals;
where I needed to decide rather than describe I used the combined
$5\times2$cv $F$ test, and with twenty-four datasets and three comparisons the
reader should expect a few values under $0.05$ by chance, as
Table~\ref{tab:direction} shows.

The residual-directed initialisation is one of several ways to choose a
direction. A second-order splitting direction in the manner of
\cite{liu2019splitting} is another, and I have not tried it; given that a
first-order data-driven direction did no better than noise, I would not expect
it to change the accuracy either, but that is an expectation and not a
measurement.

\section{Related work}

Soft decision trees are due to \.Irsoy, Y\i{}ld\i{}z and Alpayd\i{}n
\cite{irsoy2012soft}, and were revisited by Frosst and Hinton
\cite{frosst2017distilling} as a way of distilling a neural network into a model
whose decisions can be read. The hierarchical mixture of experts with subtree
dropout \cite{irsoy2021dropout} is the closely related architecture in which
leaves are themselves models. Omnivariate trees are due to Y\i{}ld\i{}z and
Alpayd\i{}n \cite{yildiz2001omnivariate}. Constructive learning of the kind
studied in Sections 4 and 5 goes back to cascade-correlation
\cite{fahlman1990cascade} and GAL \cite{alpaydin1994gal}.

The idea of enlarging a network while preserving the function it computes is the
subject of Net2Net \cite{chen2016net2net} and network morphism
\cite{wei2016network}. Both note that an exactly preserving transformation
creates symmetric parameters, and Net2Net adds noise for that reason. Splitting
steepest descent \cite{liu2019splitting} treats the same situation as an
optimisation problem: a duplicated neuron sits at a saddle in the enlarged space,
and the paper derives the second-order direction that leaves it, with a signed
refinement in later work. GradMax \cite{evci2022gradmax} keeps the function
unchanged by zeroing the new weights and then chooses the remaining freedom to
maximise the gradient norm, so that the new unit starts with the strongest
available learning signal.

Proposition~\ref{prop:zero} is the corresponding statement for tree gates, and it
is a stronger degeneracy than any of those. In the neuron case the first-order
information is uninformative and second-order information recovers a direction.
In the gate case, with identical children, the first derivative with respect to
the gate is exactly zero and the update rule reproduces the condition that made
it zero, so no amount of further first-order training escapes it. The statistical
machinery in Section~6 is the combined $5\times2$cv $F$ test
\cite{alpaydin1999combined}.

\section{Conclusion}

Growth operations that look safest are not. Preserving a soft tree's function
exactly while deepening it freezes the tree at the width it started with, and
the cost of that is zero on two-class problems and tens of points on
multi-class ones, in proportion to how many leaves the problem needs and the
tree can no longer grow. The same defect, at a single node, was responsible
for per-leaf growth losing to complete trees; with the children inheriting
the parent it no longer does, at a fraction of the splits.

How the symmetry is broken does not matter. Noise and a residual-directed
initialisation come out the same across twenty-four datasets. What matters is
that the new parameters have a gradient on the first step, which is a cheap
thing to check, and which this library now checks in a test.

\section*{Code and data availability}

The implementation, the measurement script that produced every number in this
paper, and the regression test asserting the zero gradient of
Proposition~\ref{prop:zero} are in the repository:

\begin{center}
\url{https://github.com/cgrtml/neural-trees}
\end{center}

\noindent The exact version used here is v0.7.0, archived at

\begin{center}
\url{https://doi.org/10.5281/zenodo.22929121}
\end{center}

\noindent All datasets are the standard versions distributed with scikit-learn.

\end{document}